\documentclass{article}

\usepackage[preprint]{neurips_2025}

\usepackage[utf8]{inputenc} % allow utf-8 input
\usepackage[T1]{fontenc}    % use 8-bit T1 fonts
\usepackage{hyperref}       % hyperlinks
\usepackage{url}            % simple URL typesetting
\usepackage{booktabs}       % professional-quality tables
\usepackage{amsfonts}       % blackboard math symbols
\usepackage{nicefrac}       % compact symbols for 1/2, etc.
\usepackage{microtype}      % microtypography
\usepackage{xcolor}         % colors

\usepackage{amsmath, amssymb, amsthm}

\usepackage{graphicx}
\usepackage{floatrow} %for captions on side of figure
\theoremstyle{plain}
\newtheorem{theorem}{Theorem}[section]
\newtheorem{lemma}[theorem]{Lemma}

\newtheorem{proposition}[theorem]{Proposition}
 
\theoremstyle{definition}
\newtheorem{definition}[theorem]{Definition}

\theoremstyle{remark}

\newcommand{\lms}{ \{  }
\newcommand{\rms}{ \}  }
\newcommand{\RR}{\mathbb{R}}
\newcommand{\M}{\mathcal{M}}

\title{On the (Non) Injectivity of Piecewise Linear Janossy Pooling}

\author{%
  Ilai~Reshef\\
  Faculty of Computer Science\\
  Technion - Israel Institute of Technology\\
  Haifa, Israel \\
  \texttt{ilai.reshef@campus.technion.ac.il} \\
  \And
  Nadav~Dym \\
  Faculty of Mathematics\\
  Faculty of Computer Science\\
  Technion - Israel Institute of Technology\\
  Haifa, Israel \\
    \texttt{nadavdym@technion.ac.il} \\
}

\begin{document}

\maketitle

\begin{abstract}
  %\textbf{Option 1:} This paper demonstrates a fundamental limitation in the expressive power of \(k\)-ary Janossy pooling when applied to continuous piecewise linear (CPwL) functions. We prove that if the input domain contains a line segment and the number of inputs n is greater than the pooling arity k, then the k-ary Janossy pooling of any CPwL function \(f:(\mathbb{R}^d)^k\to \mathbb{R}^m\) is not injective with respect to multisets. This means that distinct multisets of n elements can yield identical outputs after pooling, indicating an inherent loss of information about the input multiset's structure. Our findings highlight a crucial property of Janossy pooling for CPwL functions, relevant for applications relying on symmetric aggregation while needing to distinguish unique input sets.\\
  %\textbf{Option 2:} We investigate a core property of \(k\)-ary Janossy pooling, a technique employed in machine learning for processing unordered data. Our analysis demonstrates a fundamental limitation: the non-injectivity of \(k\)-ary Janossy pooling for continuous piecewise linear (CPwL) functions. This result implies that distinct input multisets can map to the same representation, which has implications for tasks where distinguishing between inputs is crucial. Future work may explore alternative pooling mechanisms or input representations that overcome this limitation and enhance discriminative power.

Multiset functions, which are functions that map multisets to vectors, are a fundamental tool in the construction of neural networks for multisets and graphs. To guarantee that the vector representation of the multiset is faithful, it is often desirable to have multiset mappings that are both injective and bi-Lipschitz. Currently, there are several constructions of multiset functions achieving both these guarantees, leading to improved performance in some tasks but often also to higher compute time than standard constructions. Accordingly, it is natural to inquire whether simpler multiset functions achieving the same guarantees are available. In this paper, we make a large step towards giving a negative answer to this question. We consider the family of \(k\)-ary Janossy pooling, which includes many of the most popular multiset models, and prove that no piecewise linear Janossy pooling function can be injective. On the positive side, we show that when restricted to multisets without multiplicities, even simple deep-sets models suffice for injectivity and bi-Lipschitzness. 
\end{abstract}

\section{Introduction}
A natural requirement of  machine learning models for graphs and point clouds is that they respect the permutation symmetries of the data. A key tool to achieve this is the process of mapping multisets, which are unordered collections of vectors,  to a single (ordered) vector which faithfully represents the multiset. 

The celebrated deep-sets paper \cite{deepsets} proposed a  simple and popular method to map multisets to vectors  via elementwise application of a function $f$, followed by  sum pooling, namely
\begin{equation}\label{eq:deepsets}
F(\lms \mathbf{x}_1,\ldots,\mathbf{x}_n \rms)=\sum_{j=1}^n f(\mathbf{x}_j) .
\end{equation}
Another popular alternative, which is more computationally demanding but also more expressive \citep{zweig2022exponential}, sums a function $f(x_i,x_j)$ over all pairs of points 
\begin{equation}\label{eq:2ary}
F(\lms \mathbf{x}_1,\ldots,\mathbf{x}_n \rms)=\sum_{i,j=1}^n f(\mathbf{x}_i,\mathbf{x}_j) 
\end{equation}
This type of pairwise summation allows incorporation of relational pooling \citep{relational}, or attention mechanisms as proposed in the set-transformer paper \citep{lee2019set}. 

A natural generalization of both these models is the notion of $k$-ary Janossy pooling, where a function $f$ is applied to all $k$-tuples of the multiset and then summation is applied to all these $k$-tuples. Deep sets model correspond to the case $k=1$ while set transformers correspond to $k=2$. Janossy pooling for general $k$ was successfully used in \cite{murphy2018janossy}. 

To ensure the quality of the vector representation of the multiset, a common requirement is that the function $F$ is injective. This requirement enables construction of maximally expressive message passing neural networks \citep{xu2018powerful,morris2019weisfeiler}, and is exploited in a variety of other scenarios where expressivity of graph neural networks is analyzed \citep{maron2019provably,hordan2024weisfeiler,sverdlov2025on,zhang2024expressive}

The inejctivity  requirement can be satisfied even by deepsets models, providing that the function $f:\RR^d \to \RR^m$  in \eqref{eq:deepsets} is defined correctly, and the embedding dimension $m$ is large enough. The various aspects of this question are discussed in \cite{wagstaff2022universal,deepsets,xu2018powerful,amir2023neural,tabaghi2024universal,wang2024polynomial}. Most relevant to our discussion are the recent results which show that \eqref{eq:deepsets} can be injective when $f$ is a neural network with smooth activations, but can never be injective when $f$ is a \emph{Continuous Piecewise Linear (CPwL)} function (as is the case when $f$ is a  neural network with ReLU activations).

While these theoretical results seem to indicate an advantage of smooth functions $f$ over CPwL ones, empirical evidence indicates that the separation between multisets via smooth activations can be very weak \cite{pmlr-v235-bravo24a,hordan2024complete}, and that empirically the separation obtained even by non-injective CPwL deep sets model is often  preferable. Thus, recent papers have argued that a more refined notion of separation is necessary, via the notion of bi-Lipschitz stability \cite{davidson2025on,amir2025fourier,balan2022permutation}. In this notion, multisets are required not only to be mapped to distinct vectors, but we also require that the distance between the vector representations resembles the natural Wasserstien distance between the multisets. 

In the lens of bi-Lipschitz stablility, the ranking of CPwL and smooth multiset functions are reversed. In fact, \cite{amir2023neural} and \cite{cahill2024bilipschitz} showed that smooth multiset functions can never be bi-Lipshitz. In contrast, while CPwL deepsets functions are not injective (or bi-Lipschitz),  several recent papers have suggested new CPwL multiset-to-vector mappings based on sorting \citep{balan2022permutation,dymGortler,balan2023ginvariant}, Fourier sampling of the quantile function \cite{amir2025fourier}, or max filters \cite{cahill2022group}, and  showed that they are both injective and bi-Lipschitz. In fact, \cite{sverdlov2024fswgnn} showed that  CPwL multiset functions which are injective are automatically also bi-Lipschitz.

Experimentally, it was shown that  these CPwL bi-Lipschitz multiset mappings have significant advantages over standard methods, for tasks like learning Wasserstein distances or learning in a low parameter regime \citep{amir2025fourier} and for graph learning tasks \cite{davidson2025on} including reduction of oversquashing \cite{sverdlov2024fswgnn}. On the other hand, these methods are typically more time consuming than standard methods, and at least at the time this paper is written they are not as prevalent as deepsets and set transformers. 

The goal of this paper is to address the following question

\textbf{Main Question:}  Is it possible to construct CPwL injective (and bi-Lipschitz) functions via $k$-ary pooling?

 Currenlty, we have a negative answer to this question only in the special case where $k=1$ (deepsets), but for $k\geq 2$ (e.g. set transformers) the answer is unknwon. A positive answer to this question would potentially lead to new bi-Lipschitz models which are closer to established models like set transformers, and potentially would have better performance. A negative answer would indicate that bi-Lipschitz models do require different types of multiset functions, such as the sort based functions currenly suggested in the literature.

\paragraph{Main Results}
Our theoretical analysis reveals two main results. The first result is that in full generality, the answer  to our main question is negative: $k$-ary Janossy pooling is not injective, except in the unrealistic scenario where $k$ is equal to the full cardinality of the multiset.  This result strengthens the argument for using sorting based bi-Lipchitz mappings as in \citep{davidson2025on,amir2025fourier,sverdlov2024fswgnn}.

Our second result shows that the obstruction to injectivity is only the existence of multisets with repeated points: on a compact domain $D$ of multisets where all multisets have distinct points, even the $1$-ary  CPwL Janossy pooling (deepsets) can be injective and bi-Lipschitz. The computational burden of this construction strongly depends on a constant $R(D)$ which measures how close multisets in $D$ are to have repeated points.   This result suggests that the advantages provided by sort-based methods may only be relevant in datasets where (near) point multiplicity occurs (e.g. point cloud samples of surfaces where points are very close together), and not in datasets where points are typically fairly far away, such as multisets which describe small molecules.  

\begin{figure}[h]
    \centering
    \includegraphics[width=0.6\columnwidth]{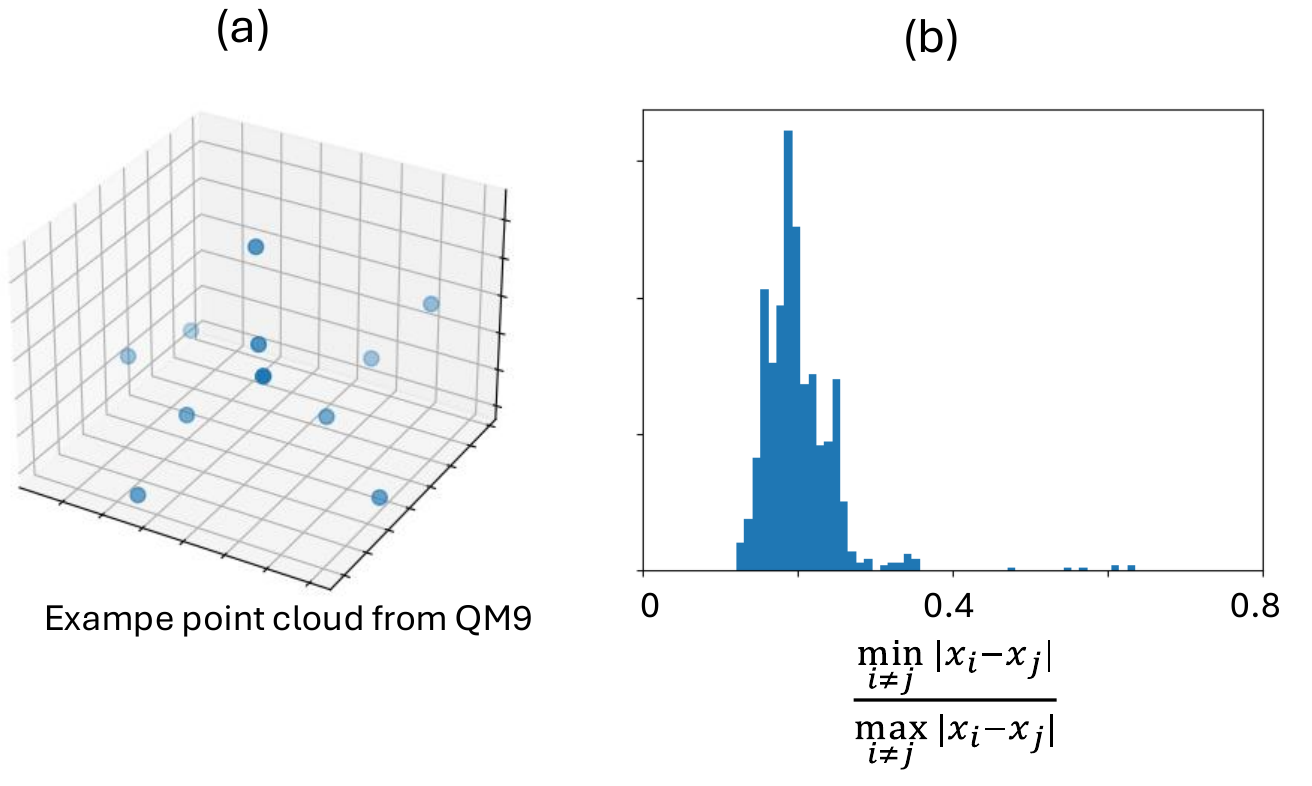}
    \caption{\it This figure illustrates that the assumption that multisets do not have (near)-repeated points is realistic for small molecule datasets. In (a) we show an example multiset from the QM9 \citep{qm9, ramakrishnan2014quantum} small molecule dataset. This example shows visually that different set elements are not very close together. In (b) we see the statistics of the minimal distance within each multiset (normalized by the maximal distance), over 1,000 represetative samples from QM9. In all these instances, the minimal distance was never lower than $0.1$.   }
    \label{fig:separable}
\end{figure}

\paragraph{Paper Structure}
In Section \ref{sec:problem} we will lay out the formal definitions needed to fully define our problem. In Section \ref{sec:non} we state and prove the non-injectivity of CPwL Janossy pooling for general domains, and in Section \ref{sec:yes} we state and prove the injectivity results for domains with non-repeated elements. Conclusions, Limitations and future work are discussed in Section \ref{sec:future}. 

\section{Problem Statement}\label{sec:problem}
We begin by formally stating the notions necessary to define our problem. For arbitrary sets \(C,Y\),  we say that a function \(F:C^n \to Y\) is permutation invariant if  \(F(\mathbf{w}_1,\ldots,\mathbf{w}_n) = F(\mathbf{w}_{\pi(1)}, \ldots,\mathbf{w}_{\pi(n)})\) for every permutation \(\pi\) of the coordinates of \(\mathbf{w} \in C^n\). 

The notion of permutation invariant functions is closely linked to the notion of functions on multisets. A multiset $\lms \mathbf{w}_1,\ldots,\mathbf{w}_n \rms$ is a collection of elements which is unordered (like sets), but where repetitions are allowed (unlike sets). We denote the space of multisets by  \(\mathcal{M}_n(C)\).  

If \(F\) is permutation invariant, we can identify it with a function on multisets in \(\M_n(C) \) via 
\[F\left( \lms \mathbf{w}_1,\ldots,\mathbf{w}_n \rms  \right)=F(\mathbf{w}_1,\ldots,\mathbf{w}_n) \]
Since $F$ is permutation invariant, this expression is well defined and does not depend on the ordering. Conversely, any multiset function $F$ on $\M_n(C)$ can be used to define a permutation invariant function on $C^n$. Due to this identification, we will use the term 'multiset function' and 'permutation invariant function' alternatingly, according to convenience. 

In this paper, our main focus is on permutation invariant functions defined by \(k\)-ary Janossy pooling. Namely, for some natural numbers $k\leq n$, and a function $f:C^k \to Y $, we define a permutation invariant function $F:C^n \to Y$ via
\begin{equation} \label{eq:janossy} F(\mathbf{x}_1,\dots,\mathbf{x}_n) = \frac{1}{(n-k)!}\sum_{\pi \in S_n} f \left(\mathbf{x}_{\pi(1)},\dots,\mathbf{x}_{\pi(k)} \right) \end{equation}
As mentioned above, special cases of Janossy pooling include the deep sets model in \eqref{eq:deepsets}, which corresponds to the case $k=1$, and set transformer models which correspond to the case $k=2$ \eqref{eq:2ary}.

As discussed in the introduction, we will focus on the case where the function $f$ used to define the Janossy pooling is \emph{Continuous Piecewise Linear (CPwL)}. To define this, we recall that a (closed, convex) polytope $P$  is a subset of $\RR^d$ defined by a
finite number of weak inequalities \[P=\{\mathbf{x}\in \RR^d| \mathbf{a}_j\cdot \mathbf{x}+b_j\geq 0, \forall j=1,\ldots,J \}\]
A partition of $\RR^d$ is a finite collection of polytopes with non-empty interior,  whose union covers $\RR^d$ and whose interiors do not intersect. A CPwL 
function $f:\RR^d \to \RR^m$ is a continuous function satisfying that, for some partition 
$\mathcal{P}=\{P_1,\ldots,P_k \}$, the restriction of $f$ to each polytope $P_j$ in the partition is an affine function. The polytopes $P_j$ are called linear regions of $f$.  Neural networks defined by piecewise linear activations like ReLU or leaky ReLU are important examples of CPwL functions. 	

Finally, a multiset function $F:\M_n(C)\to Y $ is \emph{injective} if it is injective in the standard sense: for all distinct multisets $W,W'\in \M_n(C)$ we have  $F(W)\neq F(W')$. As discussed in the introduction,  the question we discuss in this paper is the injectivity of $F$ induced from Janossy pooling of a CPwL function $f$.

%\begin{definition} %[Multiset Notation]
%	Let \(C\) be a set. We denote by \(\mathcal{M}_n(C)\) the set of all multisets of size \(n\) with elements from \(C\).
%\end{definition}

%\begin{definition}  %[Injectivity on Multisets]
%	Let \(F: C^n \rightarrow Y\) be a function. We say that \(F\) is a \emph{multiset function} if it is permutation invariant, i.e., \(F(w) = F(\pi w)\) for every permutation \(\pi\) of the coordinates of \(w \in C^n\). We say that \(F\) is \emph{injective on} \(\mathcal{M}_n(C)\) if, for all \(w, y \in C^n\), \(F(w) = F(y)\) implies that \(w\) is a permutation of \(y\). \end{definition}

\section{Non-injectivity of Janossy Pooling for general domains}\label{sec:non}

Now we can state our main theorem:

\begin{theorem}\label{theorem:non-injective-janossy}[Non-Injectivity of \(k\)-ary Janossy Pooling of CPwL functions]
Let \(C\) be a subset of \( \mathbb{R}^d\) that contains a line segment (usually this will be \([0,1]^d\) or \(\mathbb{R}^d\) itself). Let \(f:(\mathbb{R}^d)^k \rightarrow \mathbb{R}^m\) be a continuous piecewise linear (CPwL) function. Let \(n>k\), and let \(F:(\mathbb{R}^d)^n \rightarrow \mathbb{R}^m\) be the \(k\)-ary Janossy pooling of \(f\). Then \(F\) is not injective on \(\mathcal{M}_n(C)\).
\end{theorem}

To provide intuition for the theorem, we recall the simple proof for the simple case $k=1,d=1$, provided in \cite{amir2023neural}. In this case, we find a pair of distinct points $x,y$ which are in the same linear region of $f$. In this case, the average of $x$ and $y$ is also in the same linear region, and we can use this to obtain a contradiction to injectivity 
$$F(x,y)=f(x)+f(y)=f(\frac{x+y}{2})+f(\frac{x+y}{2})=F(\frac{x+y}{2},\frac{x+y}{2}) .$$

The proof of the case $k=1,d=1$ relies on the trivial observation that we can always find a pair of numbers $(x,y)\in \RR^2$ (or more generally in $\RR^n$) whose elements are distinct, but come from the same partition which the CPwL function $f:\RR\to \RR^m$ is subordinate to. To generalize this result to $k$-ary pooling, we will need to show a similar but much stronger property: for any polytope partition of $\RR^k$, one can only find a vector in $\RR^n$ of distinct monotonely decreasing elements, such that all $k$-ary montonely ordered subvectors belong to a single polytope from the partition:

\begin{theorem}
\label{proposition:n-choose-k-points-in-one-polytope}
For every polytope partion $\mathcal{P}$ of $\RR^k$,  there exists a polytope \(P_0 \in \mathcal{P}\) and a point \(\mathbf{w} = (w_1,\dots , w_n) \in (0,1)^n\) such that \(w_1 > \dots > w_n\) and, for any ascending \(k\)-tuple of indices \(i_1<\dots<i_k\) in \([n]\), the point
	\( (w_{i_1},\dots,w_{i_k})\) is in  \( \mathrm{int}(P_0)\).
\end{theorem}
A visualization of the property described in the theorem is provided in Figure \ref{fig:property} for the special case $k=2,n=3$. 

%\nd{if time permits: move to figures not on whole line, a bit ugly..}
%\begin{figure}[h]
%    \centering
%    \includegraphics[width=0.6\columnwidth]{property.pdf}
%    \caption{\it Visualization of the property from Proposition \ref{proposition:n-choose-k-points-in-one-polytope} for a partition of $[0,1]^2$ into four squares. The points $w=(3/8,2/8,1/8)$ fulfills the conditions of the proposition, as all three $2$ dimensional ordered subvectors are in the same square (see orange dots). The vector $(7/8,6/8,1/8) $ does not fulfill the condition (see green dots) }
%    \label{fig:property}
%\end{figure}

\begin{figure}
\floatbox[{\capbeside\thisfloatsetup{capbesideposition={right,top},capbesidewidth=5cm}}]{figure}[\FBwidth]
{\caption{\it Visualization of the property from Theorem \ref{proposition:n-choose-k-points-in-one-polytope} for a partition of $[0,1]^2$ into four squares. The point $\mathbf{w}=(3/8,2/8,1/8)$ fulfills the conditions of the proposition, as all three $2$ dimensional ordered subvectors are in the same square (see orange dots). The vector $(7/8,6/8,1/8) $ does not fulfill the condition (see green dots)}\label{fig:property}}
{\includegraphics[width=7 cm]{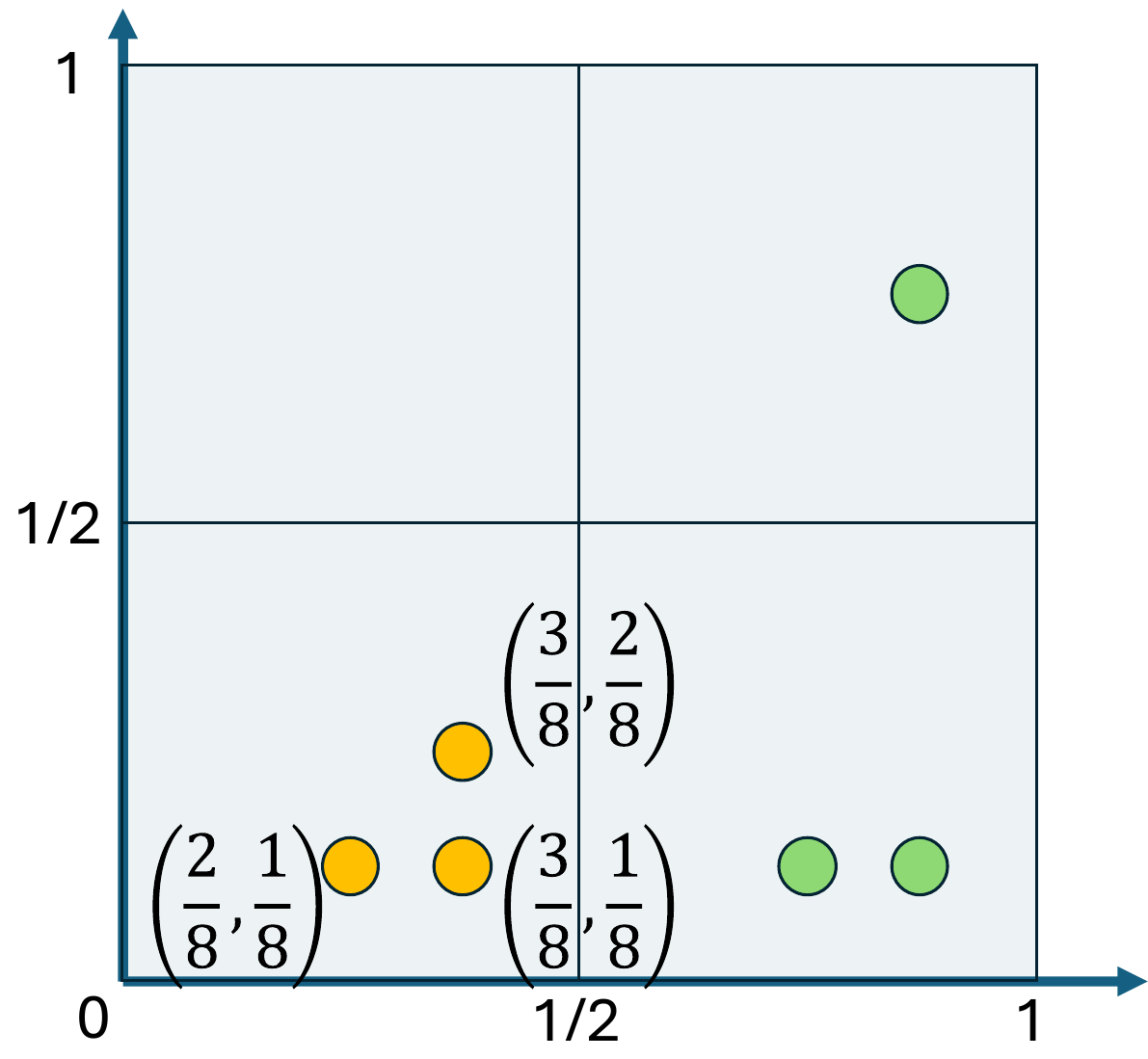}}
\end{figure}

To the best of our knowledge, Theorem \ref{proposition:n-choose-k-points-in-one-polytope} has not previously been known. The proof of this result is technical and non-trivial and it is given in the Appendix.

We now explain how this proposition can be used to prove Theorem~\ref{theorem:non-injective-janossy}.

\begin{proof}[Proof of Theorem~\ref{theorem:non-injective-janossy}]

For the sake of simplicity we first prove the theorem in the case $d=1$. We will later easily generalize this result to the general case $d>1$. We assume WLOG that \(C=[0,1]\).

Let \(f: \mathbb{R}^k \rightarrow \mathbb{R}^m\) be a CPwL function. Let \(F:\mathbb{R}^n \rightarrow \mathbb{R}^m\) be its \(k\)-ary Janossy pooling as in \eqref{eq:janossy}.  Our goal is to prove that \(F\) is not injective on multisets \( \mathcal{M}_n([0,1])\).

The first step is to show that $F$ can be defined alternatively by applying Janossy pooling to the permutation invariant function  
 \(\hat{f}:\mathbb{R}^k \rightarrow \mathbb{R}^m\) defined by
\[\hat{f}(x_1,\dots,x_k) = \sum_{\pi \in S_k} f \left(x_{\pi(1)},\dots,x_{\pi(k)} \right)\]

Note that \(\hat{f}\) is a permutation-invariant function and that we can equivalently write

\[F(x_1,\dots,x_n) = \sum_{1 \leq i_1 < \dots < i_k \leq n} \hat{f} \left(x_{i_1},\dots,x_{i_k} \right)\]

Summing over all \(\binom{n}{k}\) subsets of \([n]\) of size \(k\).

Note that \(\hat{f}\) is the sum of finitely many CPwL functions; therefore, it is itself a CPwL function. Let \(\mathcal{P}\) be a finite polytope covering of \([0,1]^k\) such that for all polytopes \(P \in \mathcal{P}\), \(\hat{f} \big|_P\) is an affine function. 
Let $P_0$ be as promised from Theorem \ref{proposition:n-choose-k-points-in-one-polytope}, and let \( A \in \mathbb{R}^{m \times k}, \mathbf{b}\in \mathbb{R}^m \) such that \(\hat{f} \big|_{P_0}(\mathbf{z}) = A\mathbf{z}+\mathbf{b}\). The properties of the point $\mathbf{w}$ in the theorem are preserved under small perturbations. Namely, for some $r>0$, we have that for all vectors $\boldsymbol\delta \in \RR^n$ with norm bounded by $r$, we will have both $w_1+\delta_1>w_2+\delta_2>\ldots> w_n+\delta_n$, and that all $k$ vectors obtained from $\mathbf{w}+\boldsymbol\delta$ by choosing $k$ different ascending indices will be in the same polytope $P_0$. It follows that for all such $\boldsymbol\delta$
\begin{align*}
F(\mathbf{w}+\boldsymbol\delta) &=  \sum_{1 \leq i_1 < \dots < i_k \leq n} \hat{f} \left(w_{i_1}+\delta_{i_1},\dots,w_{i_k}+\delta_{i_k} \right) \\
&= \binom{n}{k} \mathbf{b} + \sum_{1 \leq i_1 < \dots < i_k \leq n} A \left(w_{i_1}+\delta_{i_1},\dots,w_{i_k}+\delta_{i_k}\right)^\top \\
&= \binom{n}{k} \mathbf{b} + A \left( \sum_{1 \leq i_1 < \dots < i_k \leq n} \left(w_{i_1}+\delta_{i_1},\dots,w_{i_k}+\delta_{i_k} \right)^\top \right)
\end{align*}
To contradict injectivity we will want to obtain $F(\mathbf{w})=F(\mathbf{w}+\boldsymbol\delta)$, which will hold if 
\[\sum_{1 \leq i_1 < \dots < i_k \leq n} \left(\delta_{i_1},\dots,\delta_{i_k} \right) = \left(0,\dots,0 \right) \]
Indeed, these are $k$ linear homogenous equations in $n>k$ variables, and they have a non zero solution $\boldsymbol\delta$. We can scale this $\boldsymbol\delta$ by a sufficient small number to guarantee that $\|\boldsymbol\delta\|<r$. We then have that $F(\mathbf{w})=F(\mathbf{w}+\boldsymbol\delta)$, that $\mathbf{w}+\boldsymbol\delta\neq \mathbf{w}$, and moreover, since both $\mathbf{w}$ and $\mathbf{w}+\boldsymbol\delta$ are sorted from small to large, that $\mathbf{w}$ is not a permutation of $\mathbf{w}+\boldsymbol\delta$. Thus $F$ is not injective, and we proved the theorem in the case where $d=1$.

%Since the system admits \(w_1,\dots,w_n\) as a solution, it follows that there are infinitely many solutions. In particular, there exists a  solution \(y=(y_1,\dots,y_n) \neq w\). For all $\epsilon>0$, $y_\epsilon:=\epsilon y+(1-\epsilon) w $ and taking $\epsilon$ small enough we will have that 
%
% arbitrarily close to \(w\) in the solution space such that \(y_1 > \dots > y_n\) and \(\forall i,  |y_i - w_i|<\frac{\eta}{k}\). By the construction of \(\eta\), the points \((y_{i_1},\dots,y_{i_k})\) lie in the interior of the polytope \(P_0\) for every ascending \(k\)-tuple of indices \(i_1 < \dots < i_k\). This follows because, due to the triangle inequality, the \(\ell_1\) distance satisfies
%\[ \left\Vert \left(y_{i_1},\dots,y_{i_k} \right) - \left(w_{i_1},\dots,w_{i_k} \right) \right\Vert < \eta \]
%
%Finally, observe that \( F(y) = F(w) \), as \( F \) depends only on the sums over all \( k \)-tuples, and these sums are preserved by the construction of \( y_1, \dots, y_n \). That is,
%\begin{flalign*}
%     F(y_1, \dots, y_n)  &= \binom{n}{k} b + A \left( \sum_{1 \leq i_1 < \dots < i_k \leq n} \left(y_{i_1},\dots,y_{i_k} \right)^\top \right) \\ &= \binom{n}{k} b + A \left( \sum_{1 \leq i_1 < \dots < i_k \leq n} \left(w_{i_1},\dots,w_{i_k} \right)^\top \right) = F(w_1, \dots, w_n) \end{flalign*}
%Therefore, \(y \) is an input that yields the same output as \(w\) under \( F \). Recall that \(w \neq y\). Since both inputs are strictly sorted, they are not permutations of each other, demonstrating that \( F \) is not injective as a multiset function.

\textbf{When $d>1$.} 
We now prove the general case $d>1$ by a reduction to the case $d=1$. Let us assume by contradiction that $C \subseteq \RR^d$ contains the line segment between some (non-identical) points $\boldsymbol\alpha$ and $\boldsymbol\beta$, that $f$ is some CPwL function and that the function $F$ obtained by $k$-ary Janossy pooling on $f$ is injective.

Let \(g:[0,1] \to C\) be the affine function \(g(t)=(1-t)\boldsymbol\alpha+t\boldsymbol\beta\). For any natural $s$, we can extend  $g$  to a mapping  \(g^{(s)}:[0,1]^s \to (\mathbb{R}^d)^s\)  by applying \(g\) to each coordinate, i.e., for \(\mathbf{t} = (t_1, \dots, t_k) \in [0,1]^s\), \(g^{(s)}(\mathbf{t}) = (g(t_1), \dots, g(t_s))\). The function \(g^{(s)}\) is  affine and injective. Accordingly, the function $F\circ g^{(n)}:C^n\to \RR^m $ is injective, and we note that it is the Janossy pooling of $f\circ g^{(k)} $ which is a CPwL function as the composition of a CPwL function and an affine function. This leads to a contradiction to our proof for the case $d=1$.
\end{proof}

\subsection{Janossy pooling when $k=n$}
In the degenerate case where we use $n$-ary pooling for multisets of cardinality $n$, an expensive averaging over all permutations is necessary. In this case we can choose the initial $f$ we use to be a CPwL multiset injective function, such as the sorting based functions constructed in \cite{balan2022permutation} and mentioned earlier. Since the initial $f$ is already permutation invariant, and $n=k$, we would obtain $F(\mathbf{x}_1,\ldots,\mathbf{x}_n)=f(\mathbf{x}_1,\ldots,\mathbf{x}_n) $ in this case, and so Janossy pooling of CPwL functions can be injective in this degenerate case. 

\section{Injectivity under restricted domains}\label{sec:yes}
We now discuss our second main result. In contrast to the non-injectivity result presented in the previous section for general multiset domains of the form \(\mathcal{M}_n(C)\), we now show that by restricting the domain to a compact \(D \subset \mathcal{M}_n(C)\) where each multiset has \(n\) distinct elements, injective \(1\)-ary Janossy pooling is possible. To state this theorem formally, we define the natural Wasserstein metric on the space of multiset, and then the notion of a compact set in multiset-space:

\begin{definition}%[Multiset Notation]
    Given two multisets \(A, B \in \mathcal{M}_n(\mathbb{R}^d)\), the Wasserstein metric \(d_W(A, B)\) is defined as    
\[
d_W(A, B) = \min_{\sigma \in S_n} \sum_{i=1}^n \|\mathbf{a}_i - \mathbf{b}_{\sigma(i)}\|
\]    
where \(A = \{\mathbf{a}_1, \mathbf{a}_2, \dots, \mathbf{a}_n\}\), \(B = \{\mathbf{b}_1, \mathbf{b}_2, \dots, \mathbf{b}_n\}\), \(S_n\) is the set of all permutations of \([n]\), and \(\|\cdot\|\) denotes the \(\ell_\infty\) norm. Note that the expression above is permutation invariant, and therefore well-defined independently of the order of the elements of \(A,B\).
\end{definition}

\begin{definition}
    Let \(C\subset \mathbb{R}^d\). We say that \(D \subset \mathcal{M}_n(C)\) is compact if every sequence of multisets \(\{A_j\}_{j=1}^{\infty}\) in \(D\) has a subsequence that converges to a multiset in \(D\) with respect to the Wasserstein metric \(d_W\). This is the standard definition of compactness in a metric space. %Equivalently, \(D\) is compact if and only if it is closed and bounded with respect to the Wasserstein metric \(d_W\).
\end{definition}

We can now state our second main result: in the absence of multisets with repeated elements, even $1$-ary pooling is injective:

\begin{theorem} \label{thm:main_injectivity}
    Let \(D \subset \mathcal{M}_n(C)\) be a compact set of multisets where each multiset has $n$ distinct elements. Then there exists some $m=m(D)$ and a continuous piecewise linear function \(f: \mathbb{R}^d \rightarrow \mathbb{R}^m\) such that its \(1\)-ary Janossy pooling \(F(A) = \sum_{\mathbf{a} \in A} f(\mathbf{a})\) is injective on \(D\), and bi-Lipschitz with respect to the Wasserstein distance.
\end{theorem}
We will prove this theorem by construction. We begin with some preliminaries: we first introduce the minimal separation function  \(r: D \rightarrow \mathbb{R}_{\geq 0}\) via
    \[
    r(A) = \min_{\substack{\mathbf{a}_i, \mathbf{a}_j \in A \\ i \neq j}} \|\mathbf{a}_i - \mathbf{a}_j\|
    \]
    for any multiset \(A = \{\mathbf{a}_1, \mathbf{a}_2, ..., \mathbf{a}_n\} \in D\), where \(\|\cdot\|\) denotes the \(\ell_\infty\) norm. We note that by our theorem's assumptions, $r(A)>0$ for all $A\in D$. We next define $R(D)$ to be the minimal separation obtained on all of $D$, namely
    \[
    R(D) = \inf_{A \in D} r(A) = \inf_{A \in D} \min_{\substack{\mathbf{a}_i, \mathbf{a}_j \in A \\ i \neq j}} \|\mathbf{a}_i - \mathbf{a}_j\|.
    \]
We next show that, due to the compactness of $D$, the infimum in the definition of $R(D)$ is obtained and $R(D)$ is always strictly positive.

    \begin{proposition} \label{prop:positive_R_D}
        If $D \subset \mathcal{M}_n(C)$ is a compact set of multisets, where each multiset $A \in D$ consists of $n$ distinct elements, then its minimum separation $R(D)$ is positive.
    \end{proposition}
\begin{proof}
    Assume, for the sake of contradiction, that \(R(D) = 0\).    
    By the definition of the infimum, this implies that there exists a sequence of multisets \(\{A_j\}_{j=1}^{\infty}\) in \(D\) such that \(r(A_j) \rightarrow 0\) as \(j \rightarrow \infty\). Each $A_j = \{\mathbf{a}_1^{(j)}, \dots, \mathbf{a}_n^{(j)}\}$ consists of $n$ distinct points.

    Since \(D\) is compact, the sequence \(\{A_j\}_{j=1}^{\infty}\) has a subsequence \(\{A_{j_l}\}_{l=1}^{\infty}\) that converges to a multiset \(A^* \in D\) with respect to the Wasserstein metric \(d_W\). Let $A^* = \{\mathbf{a}_1^*, \dots, \mathbf{a}_n^*\}$. By definition of \(D\), we have \(r(A^*)>0\).

    Let \(\epsilon = \frac{r(A^*)}{2}\). From the convergence of $r(A_{j_l})$ and $A_{j_l}$, there exists an $l$ such that $r(A_{j_l})<\epsilon $ and \( d_w(A_{j_{l}} , A^*) < \epsilon\).  For this $l$, we deduce there are at least two distinct points, WLOG $a_1^{(j_l)}, a_2^{(j_l)} \in A_{j_l}$, such that $\|a_1^{(j_l)} - a_2^{(j_l)}\| < \epsilon$.
Next,  let \(\sigma \in S_n\) be the permutation such that the minimum in the definition of the Wasserstein distance between \(A^*,A_{j_{l}}\) is attained. Then, applying the triangle inequality twice, we get:
    \begin{align*}
     \frac{r(A^*)}{2} = \epsilon
     &> d_w(A_{j_{l}},A^*)\\
     &= \sum_{i=1}^n \| \mathbf{a}^{(j_{l})}_i - \mathbf{a}^*_{\sigma(i)}\|\\
     &\geq \| \mathbf{a}^{(j_{l})}_1 - \mathbf{a}^*_{\sigma(1)}\| + \| \mathbf{a}^{(j_{l})}_2 - \mathbf{a}^*_{\sigma(2)}\|\\
     &\geq
     \|  \mathbf{a}^*_{\sigma(1)} -  \mathbf{a}^*_{\sigma(2)} +\mathbf{a}^{(j_{l})}_2 -
     \mathbf{a}^{(j_{l})}_1 \| \\
     &\geq 
     \|  \mathbf{a}^*_{\sigma(1)} -  \mathbf{a}^*_{\sigma(2)} \| -\| \mathbf{a}^{(j_{l})}_1 -
     \mathbf{a}^{(j_{l})}_2 \|\\
     &\geq r(A^*) - \epsilon = \frac{r(A^*)}{2}
     \end{align*}
    This is a contradiction. We conclude that \(R(D)>0\).
\end{proof}
We now provide the construction of the function $f$.
Tessellate $\mathbb{R}^d$ with a grid of non-overlapping, adjacent $d$-dimensional hypercubes $Q_k$, each with side length $s = \frac{R(D)}{2}$.

We define a $\delta$-margin around each hypercube $Q_k$ using the $\ell_\infty$ distance. For any point $\mathbf{x}$, its $\ell_\infty$ distance to the hypercube $Q_k$ is given by $d_\infty(\mathbf{x},Q_k) = \min_{\mathbf{y} \in Q_k} ||\mathbf{x}-\mathbf{y}||_\infty$. The $\delta$-margin of $Q_k$ is then the set of points $\{\mathbf{x} \in \mathbb{R}^d \setminus Q_k \mid d_\infty(\mathbf{x},Q_k) < \delta\}$.

Let $\delta$ be a margin width chosen such that $0 < \delta < \frac{R(D)}{4}$. This ensures that the equation $(s + 2\delta) < R(D)$ is satisfied. This implies that if a hypercube $Q_k$ together with its $\delta$-margin contains a point $\mathbf{a} \in A$ (for $A \in D$), it cannot contain any other point $\mathbf{a}' \in A \setminus \{\mathbf{a}\}$. In particular, $Q_k$ itself, can contain at most one point from \(A\).

Let $\mathcal{I}$ be the finite set of indices of hypercubes $Q_k$ that intersect $C' = \bigcup_{A \in D} A \subseteq C$. Since $D$ is compact, $C'$ is bounded, ensuring $\mathcal{I}$ is finite.

For each hypercube $Q \in \{Q_k\}_{k \in \mathcal{I}}$, we define a local $(d+1)$-dimensional feature vector $f_Q(\mathbf{x})$, consisting of two components:
% \begin{enumerate}
    % \item
    
    \textbf{Indicator Component $f_{Q,\text{ind}}(\mathbf{x}) \in [0,1]$:}
    $f_{Q,\text{ind}}(\mathbf{x}) = 1$ if $\mathbf{x} \in Q$, and $f_{Q,\text{ind}}(\mathbf{x}) = \max(0, 1 - d_\infty(\mathbf{x},Q)/\delta)$ elsewhere. This ensures $f_{Q,\text{ind}}(\mathbf{x})=1 \iff \mathbf{x} \in Q$, and that the support of $f_{Q,\text{ind}}$ is precisely $Q$ together with is \(\delta\)-margin. Note that this component is CPwL.

% \item 
\textbf{Relative Coordinate Component ($\mathbf{f}_{Q,\text{coords}}(\mathbf{x}) \in \mathbb{R}^d$):}
    This is a CPwL function defined by the following properties:
    \begin{itemize}
        \item If $\mathbf{x} \in Q$, then $\mathbf{f}_{Q,\text{coords}}(\mathbf{x}) = \mathbf{x}$.
        \item If $\mathbf{x}$ is located outside $Q$ and its $\delta$-margin (i.e., $d_\infty(\mathbf{x},Q) \ge \delta$), then $\mathbf{f}_{Q,\text{coords}}(\mathbf{x}) = \mathbf{0}$.
        \item In the $\delta$-margin (i.e., for $\mathbf{x}$ such that $0 < d_\infty(\mathbf{x},Q) < \delta$), $\mathbf{f}_{Q,\text{coords}}(\mathbf{x})$ interpolates continuously and piecewise linearly between the values at $\partial Q$, and \(\mathbf{0}\) at the outer boundary of the margin.
    \end{itemize}

We shall now demonstrate how a function satisfying the third condition can be constructed.
By \cite{Goodman1988}, the \(\delta\)-margin can be triangulated without introducing new vertices such that each simplex of the triangulation contains vertices belonging to both \(Q\) and the outer border of the margin.

It is well known that given a simplex in \(\mathbb{R}^d\) defined by \(d+1\) affinely independent points \(p_0, \ldots, p_d\), and corresponding values \(y_0, \ldots, y_d\), there exists a unique affine function \(h\) such that \(h(x_i) = y_i\) for all \(i\).

Applying this to our triangulated \(\delta\)-margin, we define \(f\) piecewise over each simplex by assigning the known values of \(f\) at its vertices—values from \(Q\) and zeros from the outer border. The unique affine interpolation over each simplex ensures that \(f\) transitions continuously between the identity on \(Q\) and zero on the outer region, satisfying the desired conditions.

Constructions of this sort are standard in numerical analysis and finite element methods. For instance, in the context of simplicial finite elements, the \(\mathbb{P}_1\) interpolant of a function \(v\) is the unique piecewise affine function that coincides with \(v\) at the mesh vertices (see, e.g., \cite[3.3]{BrennerScott2008}).

    The function $f_Q(\mathbf{x}) = (f_{Q,\text{ind}}(\mathbf{x}),\mathbf{f}_{Q,\text{coords}}(\mathbf{x}))$ is therefore CPwL.
% \end{enumerate}

The overall function $f: \mathbb{R}^d \rightarrow \mathbb{R}^m$ is the concatenation $f(\mathbf{x}) = ( \dots, f_{Q_k}(\mathbf{x}), \dots )_{k \in \mathcal{I}}$. The output dimension is $m = |\mathcal{I}| \cdot (d+1)$.

Now that we have defined the CPwL function $f$ we will use for the proof, we formally conclude the proof: 
\begin{proof}[Proof of Theorem~\ref{thm:main_injectivity}]
Let $A \in D$ be a multiset $\{ \mathbf{a}_1, \dots, \mathbf{a}_n \}$. The \(1\)-ary Janossy pooling is $F(A) = \sum_{j=1}^n f(\mathbf{a}_j)$. We show $A$ can be uniquely recovered from $F(A)$. Let $F_{Q_k, \text{ind}}$ and $\mathbf{F}_{Q_k, \text{coords}}$ be the components of $F(A)$ corresponding to $Q_k$. We first prove two simple lemmas

\begin{lemma} \label{prop:indicator-works-properly}
    \(F_{Q_k, \text{ind}}(A) = 1\) if and only if there exists a unique \(\mathbf{a}\in A\) such that \(\mathbf{a} \in Q_k\).
\end{lemma}
\begin{proof}
    \textbf{(\(\Rightarrow\))} Suppose \(F_{Q_k, \text{ind}}(A) = 1\). Assume for the sake of contradiction that there is no \(\mathbf{a}\in A\) such that \(\mathbf{a} \in Q_k\). The support of \(f_{Q,\text{ind}}\) is \(Q_k\) together with its \(\delta\)-margin. For all \(\mathbf{a}\) in this margin, \(0<f_{Q,\text{ind}}(\mathbf{a}) < 1\). However, \[F_{Q_k, \text{ind}}(A) = \sum_{j=1}^n f_{Q_k,\text{ind}}(\mathbf{a}_j) = 1\]
    This implies that at least two elements of \(A\) lie in the \(\delta\)-margin of \(Q_k\), which is a contradiction to the separation condition that \(\delta\) was constructed to satisfy.

    \textbf{(\(\Leftarrow\))} Suppose there exists \(\mathbf{a} \in A\) such that \(\mathbf{a} \in Q_k\). By construction of \(s\) and \(\delta\), no other element of \(A\) lies in the support of \(f_{Q,\text{ind}}\). Therefore, 
    \[F_{Q_k, \text{ind}}(A) = f_{Q_k,\text{ind}}(\mathbf{a}) = 1\]
\end{proof}

\begin{lemma}\label{prop:coords-works-properly}
    If \(\mathbf{a} \in A\) lies in a hypercube \(Q_k\), then \[\mathbf{F}_{Q_k, \text{coords}} (A) = \mathbf{f}_{Q_k,\text{coords}}(\mathbf{a}) = \mathbf{a} \]
\end{lemma}

\begin{proof}
    The proof is similar to the direction \textbf{(\(\Leftarrow\))} in the proof of~\ref{prop:indicator-works-properly}.
\end{proof}

We now prove that \(A\) can be recovered uniquely from \(F(A)\). We do this using the following procedure. We go over all hypercubes $Q_k$. We then check whether \(F_{Q_k, \text{ind}}(A) = 1\). By  Lemma\ref{prop:indicator-works-properly} we know that this is the case if and only if $A$ contained an element in $Q_k$, and in this case the element is unique. We can now recover this element from $F_{Q_k,\text{coords}} $ using Lemma~\ref{prop:coords-works-properly}. We have thus uniquely recovered all elements of $A$. We note that if $A$ contains elements which are in the intersection of several hypercubes, this reconstruction procedure will give us the same elements of $A$ from several different hypercubes. This does not cause any issues since we know that $A$ does not contain multiplicities. 

Finally, to prove the bi-Lipschitzness of the construction: we note that the set $D$ could be covered by a finite union of polytopes (e.g. hypercubes) so that the union of all these hypercubes $\hat D$ contains $D$ but still does not contain multisets with repeated elments. As we now proved, we can construct a CPwL function $f$ so that the resulting $F$ obtained from 1-ary Janossy pooling will be injective on $\hat D$. Since $F$ and the 1-Wasserstein distance are both CPwL functions which attain the same zeros on $\hat D \times \hat D $, and $\hat D \times \hat D$ can be written a a finite union of compact polytopes,  We can apply \cite[Lemma 3.4]{sverdlov2024fswgnn} to show that $F$ is bi-Lipschitz on each polytope separately, and therefore also on the union which gives us $\hat D \times \hat D$. 
\end{proof}
%Recall \(A\) contains exactly \(n\) distinct elements. For all \(\mathbf{a} \in A\), \(\mathbf{a}\) lies in at least one hypercube \(Q_k\). By Lemma~\ref{prop:indicator-works-properly}, \(F_{Q_k, \text{ind}}(A) = 1\). By Lemma~\ref{prop:coords-works-properly}, we can uniquely recover the coordinates of \(\mathbf{a}\). 

%In contrast, for any \(\mathbf{x}\in \mathbb{R}^d \setminus A\), if \(\mathbf{x}\) lies in the same hypercube as a vector \(\mathbf{a} \in A\), then $F(A)$ will yield the coordinates of \(\mathbf{a}\) in the relevant coordinate component. And if \(\mathbf{x}\) does not lie in the same hypercube as any element of \(A\), then by Proposition~\ref{prop:indicator-works-properly} \(F(A)\) will not equal \(1\) in the indicator components of the hypercubes that contain \(\mathbf{x}\). Therefore, the recovery process will not produce points that are not elements of \(A\). This finalizes the proof of Theorem~\ref{thm:main_injectivity}.
\subsection{Dependence on separation}
We note that the dimension $m$ which $F,f$ map to in the construction,  depends strongly on the separation $R(D)$ and the dimension $d$. If we add the assumption that all elements of multisets $A$ are in the unit cube $[0,1]^d$, then a tesselation of side length $\sim R(D) $ would require an embedding dimension of $m\sim (1/R(D))^d $. This suggests that when $D$ contains elements which are 'almost identical', so that  $R(D)$ is small, then 1-ary pooling may not really be enough to get a good embedding with an affordable function $f$.  

As shortly discussed in the introduction, one possible example where the separation $R(D)$ is reasonably large is small molecules. To examine this, we randomly chose 1000 molecules from the QM9 \cite{qm9, ramakrishnan2014quantum} small molecule datasets. Each molecule is represented as a multiset of vectors residing in $\RR^3$. For each of these multisets, we computed the minimum distance between multiset elements, and normalized it by the maximal distance between elements. A histogram of the results is shown in Figure\ref{fig:separable}(b). We see that in all instances the ratio was not larger than $1/10$, so we can estimate that a ratio of $R(D)\approx 1/10$ could be reasonable for this type of problem. 

\section{Conclusion, limitations and future Work}\label{sec:future}
In this paper we showed two main results (a)  continuous piecewise linear Janossy pooling is not injective, when considering general domain,  and (b) on compact domains with non-repeated points, even $1$-ary continuous piecewise linear Janossy pooling can be  injective. These results suggest that deepsets models may be sufficient for tasks where multisets do not have multiplicities (so that they are sets), and the margin between closest points is significant. At the same time, when this margin is small it strengthens the case for using injective and bi-Lipschitz CPwL models such as \cite{davidson2025on,amir2025fourier}, since we show that alternative natural methods cannot attain similar theoretical guarantees. 

Building upon our positive result for 1-ary CPwL Janossy pooling on domains of sets (i.e., multisets with point multiplicities of at most one), a natural direction for future work is to explore the capacity of higher-order pooling. We conjecture that for a given integer $k \ge 1$, $k$-ary CPwL Janossy pooling can be injective on compact domains of multisets where the multiplicity of any individual element is at most $k$. 
% Establishing this conjecture would significantly extend the understanding of how CPwL Janossy architectures can effectively represent multisets with bounded repetitions while maintaining injectivity, thereby broadening their theoretical applicability.

A limitation of this work is that we only analyze the injectivity of CPwL Janossy pooling. Our focus on these functions stems from the fact that CPwL injectivity implies bi-Lipschitzness, while smooth multiset functions, which can be injective via Janossy pooling, cannot be bi-Lipschitz \cite{amir2023neural,cahill2024bilipschitz}. However, there are many functions which are neither CPwL nor smooth. An interesting avenue for future work is investigating whether such functions can be used to construct injective and bi-Lipschitz multiset functions via $k$-ary pooling, and whether these can lead to multiset models with good empirical performance. This question is most interesting for $k=2$ as $2$-ary Janossy pooling has reasonable complexity, and as for $k=1$ such a function can only exist if it is not differentiable at any point \cite{amir2023neural}.

\textbf{Acknowledgements} N.D. was supported by ISF grant 272/23.

\bibliographystyle{plainnat} % You can choose other styles like abbrvnat, unsrtnat, etc.
\bibliography{main} 

\newpage
\appendix

\section{Technical Appendix}

\subsection{Proof of Theorem~\ref{proposition:n-choose-k-points-in-one-polytope}}

% We restate the theorem:

% Let \(k,n \in \mathbb{N}\) such that \(k<n\). Let \(\mathcal{P}\) be a finite polytope covering of \([0,1]^k\). Then there exists a polytope \(P_0 \in \mathcal{P}\) and a list of \(n\) points \(w = (w_1,\dots , w_n) \in (0,1)^n\) such that \(w_1 > \dots > w_n\) and, for any ascending \(k\)-tuple of indices \(i_1<\dots<i_k\) in \([n]\), the point \(z = (w_{i_1},\dots,w_{i_k}) \in \mathrm{int}(P_0)\).

We first establish the following supporting result:

Let \(\mathbf{v} \in [0,1]^k\). We define \(\mathrm{POLY}(\mathbf{v}) = \{ P \in \mathcal{P} \mid \mathbf{v} \in P \}\) to be the set of polytopes in the covering \(\mathcal{P}\) that contain the point \(\mathbf{v}\).

\begin{lemma}\label{lemma:ball-in-all-polytopes}
Let \(\mathbf{v} \in \mathbb{R}^k\). Let \(\mathcal{P}\) be a finite polytope covering of \(\mathbb{R}^k\). There exists \(\epsilon>0\) such that the \(\epsilon\)-ball around \(\mathbf{v}\) w.r.t. the \(\ell_1\) metric, denoted by \(B_{\ell_1}(\mathbf{v}, \epsilon) = \{ \mathbf{x} \in [0,1]^k : \|\mathbf{x} - \mathbf{v}\|_1 < \epsilon \}\), does not intersect any polytope that does not contain \(\mathbf{v}\):

\[ B_{\ell_1}(\mathbf{v}, \epsilon) \cap \bigcup (\mathcal{P} \setminus \mathrm{POLY}(\mathbf{v})) = \emptyset \]

\end{lemma}

\begin{proof}[Proof of Lemma~\ref{lemma:ball-in-all-polytopes}]
Let \(P \in \mathcal{P}\setminus POLY(\mathbf{v})\). Since \(P\) is closed, \[\mathrm{dist}(\mathbf{v},P) = \inf_{\mathbf{x} \in P} \left\Vert \mathbf{x}-\mathbf{v} \right\Vert > 0\]
Let
\[\epsilon  = \frac{1}{2} \min_{P \in \mathcal{P}\setminus \mathrm{POLY}(\mathbf{v})} \left(\mathrm{dist}(\mathbf{v},P) \right)\]
Since \(\mathcal{P}\) is finite, \(\epsilon\) is well defined and positive. For this choice of \(\epsilon\), we have
\[ B_{\ell_1}(\mathbf{v}, \epsilon) \cap \bigcup (\mathcal{P} \setminus \mathrm{POLY}(\mathbf{v})) = \emptyset \]\end{proof}

% Now on to the proof of the theorem.

% \begin{proposition}
%     There exists a point \(v_0 = (x,\dots,x)\in \Delta_k\) such that for all \(P \in \mathrm{POLY}(v_0)\), \(v_0\) is an interior point of \(P \cap \Delta_k\) w.r.t. the subspace topology induced by the Euclidean topology on \(\Delta_k\).
% \end{proposition}

% \begin{proof}
%     Should I prove this here on in the lemmas?
% \end{proof}

\begin{proof}[Proof of Theorem~\ref{proposition:n-choose-k-points-in-one-polytope}]
Fix some $x\in (0,1)$ and let \(\mathbf{v}_0 = (x,\dots,x) \in (0,1)^k\).

Using Lemma~\ref{lemma:ball-in-all-polytopes}, let  \(\epsilon_1>0\) such that \(B_{\ell_1}(\mathbf{v}_0,\epsilon_1) \subset \bigcup \mathrm{POLY}(\mathbf{v}_0)\) and \(B_{\ell_1}(\mathbf{v}_0,\epsilon_1) \subset (0,1)^k\).

Let \(\mathbf{v}_1 = \mathbf{v}_0+\frac{\epsilon_1}{2} \mathbf{e}_1 = (x+\frac{\epsilon_1}{2}, x, \dots, x)\).

We continue this construction by an inductive process. For all \(1<i \leq k\):

Let \(\epsilon_i>0\) such that \(B_{\ell_1}(\mathbf{v}_{i-1},\epsilon_i) \subset \bigcup \mathrm{POLY}(\mathbf{v}_{i-1})\) and \(\epsilon_i < \frac{\epsilon_{i-1}}{2}\).

Let \(\mathbf{v}_i = \mathbf{v}_{i-1}+\mathbf{e}_i \frac{\epsilon_i}{2} = (x+\frac{\epsilon_1}{2}, \dots, x+\frac{\epsilon_i}{2}, x, \dots, x)\).

In the end of this process we get a sequence of \(k+1\) vectors $\mathbf{v}_0,\ldots,\mathbf{v}_k$ which are all in \(\mathbb{R}^k_{\text{sorted}}:=\{\mathbf{y}\in \RR^k| \quad y_1\geq y_2 \geq \ldots \geq y_k \}\).

\begin{proposition}\label{proposition:poly-inclusion-order}
    \(\mathrm{POLY}(\mathbf{v}_k) \subset \dots \subset \mathrm{POLY}(\mathbf{v}_0)\)
\end{proposition}
\begin{proof}
    Let \(i \in [k]\). Note that \(\mathbf{v}_i - \mathbf{v}_{i-1} = \frac{\epsilon_i}{2} \mathbf{e}_i \Rightarrow \mathbf{v}_i \in B_{\ell_1}(\mathbf{v}_{i-1},\epsilon_i)\). Assume, for the sake of contradiction, that there exists \(P \not\in \mathrm{POLY}(\mathbf{v}_{i-1})\) such that \(\mathbf{v}_i \in P\). Then:

    \[ \mathbf{v}_i \in  B_{\ell_1}(\mathbf{v}_{i-1}, \epsilon_i) \cap \bigcup (\mathcal{P} \setminus \mathrm{POLY}(\mathbf{v}_{i-1})) = \emptyset\]

Which is a contradiction.
\end{proof}
\begin{proposition}\label{proposition:all-v-in-same-polytope}
    There exists a single polytope \(P_0 \in \mathcal{P}\) such that \(\mathbf{v}_k \in P_0\); in particular, \(\mathbf{v}_k\) lies in the interior of this polytope.
\end{proposition}
\begin{proof}
    Assume, for the sake of contradiction, that \(|\mathrm{POLY}(\mathbf{v}_k)|>1\). Let \(P_0,P_1 \in \mathrm{POLY}(\mathbf{v}_k)\) be two different polytopes. Convexity is preserved under intersection, therefore \(P_0 \cap P_1\) is a convex set. Under the assumption that the interiors of polytopes in \(\mathcal{P}\) do not intersect, we see that \(P_0 \cap P_1\) has an empty interior; therefore, there exists a hyperplane \(H \subset \mathbb{R}^k\) such that \(P_0 \cap P_1 \subset H\) (see \cite[2.5.2]{boyd2004convex}).
    
    By Proposition~\ref{proposition:poly-inclusion-order}, we have \(\mathbf{v}_0, \dots, \mathbf{v}_k \in P_0 \cap P_1 \subset H\); however, it is easy to see that \(\mathbf{v}_0, \dots, \mathbf{v}_k\) are \(k+1\) affinely independent vectors in \(\mathbb{R}^k\), and therefore do not all lie in the same hyperplane. This is a contradiction.

    We have proved that \(\mathbf{v}_k\) lies on a single polytope \(P_0\), therefore it can either lie in the interior of \(P_0\) or on the boundary of \([0,1]^k\); however, by the construction of each \(\epsilon_i\), and by the triangle inequality, it is easy to see that \(\mathbf{v}_k \in B_{\ell_1}(\mathbf{v}_0,\epsilon_1) \subset (0,1)^k = \text{int}([0,1]^k)\). We conclude that \(\mathbf{v}_k \in \text{int}(P_0)\).
\end{proof}

Let \(\delta =  \displaystyle \min_{1<i\leq k} \left(  \frac{\epsilon_i}{\epsilon_{i-1}} \right) < 1\).

\begin{proposition}\label{proposition:constructing-interior-points-alt}
    The interior of \(P_0\) contains all points of the form \((x+y_1, \dots, x+y_k)\) for which:
    \begin{enumerate}
    \item[(a)] All \(y_i\) are positive, and are smaller than \(\frac{\epsilon_1}{2}\).
    \item[(b)] The ratio between \(y_{i+1}\) and \(y_i\) is smaller than or equal to \(\delta\).
\end{enumerate}
Moreover, each such point is in $\mathbb{R}^k_{\text{sorted}}$.
\end{proposition}
\begin{proof}
    By Proposition~\ref{proposition:poly-inclusion-order}, \(\mathbf{v}_0, \dots, \mathbf{v}_k \in P_0\). We will show that \((x+y_1, \dots, x+y_k)\) is a convex combination of the points \(v_0, \dots, v_k\) by finding appropriate coefficients.

        Let
    \begin{align*}
        \alpha_0 &= 1-\frac{2y_1}{\epsilon_1} \\
        \alpha_i &= \frac{2y_i}{\epsilon_i} - \frac{2y_{i+1}}{\epsilon_{i+1}} \quad \text{for} \quad 1\leq i<k\\
        \alpha_k &= \frac{2y_k}{\epsilon_k}
    \end{align*}
        First, we show that the sum of these coefficients equals \(1\):
    \[\sum_{i=0}^k  \alpha_i = \left(1-\frac{2y_1}{\epsilon_1}\right) + 
     \sum_{i=1}^{k-1} \left( \frac{2y_i}{\epsilon_i} - \frac{2y_{i+1}}{\epsilon_{i+1}} \right) + \frac{2y_k}{\epsilon_k}\]

     Notice that the terms in the summation telescope, as each \(- \frac{2y_{i+1}}{\epsilon_{i+1}}\) cancels with the corresponding \(\frac{2y_i}{\epsilon_i}\) from the next term. After cancellation, we are left with: \[ 1-\frac{2y_1}{\epsilon_1} + \frac{2y_1}{\epsilon_1}  -\frac{2y_k}{\epsilon_k} + \frac{2y_k}{\epsilon_k} = 1\]

         Second, we show that all these coefficients are nonnegative.  Clearly \(\alpha_0,\alpha_k >0\). 
     For \(1 \leq i < k\), we have:
\[ \frac{\epsilon_{i+1}}{\epsilon_i} \geq \delta \geq \frac{y_{i+1}}{y_i}\\      \] Where the RHS holds due to condition (b) on \(y_i,y_{i+1}\), and the LHS holds from the definition of \(\delta\). Consequently, \[ \alpha_i = \frac{2y_i}{\epsilon_i} - \frac{2y_{i+1}}{\epsilon_{i+1}} \geq 0      \] 

Third, we show that:
    \[\sum_{i=0}^k  \alpha_i v_i =  (x+y_1,\dots,x+y_k)\]
    Let's fix any coordinate \(1 \leq j \leq k\). Then we obtain
    \begin{align*}
        \left\langle \mathbf{e}_j, \sum_{i=0}^k  \alpha_i \mathbf{v}_i \right\rangle 
        &=
        \sum_{i=0}^k \alpha_i \langle \mathbf{e}_j , \mathbf{v}_i \rangle 
        =
        \sum_{i=0}^k \alpha_i x + \sum_{i=j}^k \alpha_i \frac{\epsilon_j}{2} \\
        &=
        %change from here
        x + \sum_{i=j}^{k-1} \left[ \left(\frac{2y_i}{\epsilon_i} - \frac{2y_{i+1}}{\epsilon_{i+1}} \right) \frac{\epsilon_j}{2} \right] + \left(\frac{2y_k}{\epsilon_k}\right) \frac{\epsilon_j}{2}\\
        &= 
        x + y_j
    \end{align*}

We have proved that \( (x+y_1, \dots, x+y_k) \) is a convex combination of \(\mathbf{v}_0, \dots, \mathbf{v}_k \in P_0\). By Proposition~\ref{proposition:all-v-in-same-polytope}, \(\mathbf{v}_k \in \text{int}(P_0)\). Since the coefficient of \(\mathbf{v}_k\) in this convex combination is \(\alpha_k >0\), it follows from \cite[Theorem~6.1]{rockafellar1970convex}, often referred to as the Accessibility Lemma, that \((x+y_1, \dots, x+y_k) \in \text{int}(P_0)\).

Finally, the fact that $(x+y_1, \dots, x+y_k)\in \mathbb{R}^k_{\text{sorted}}$ follows immediately from the fact that all $y_i$ are positive and $\frac{y_{i+1}}{y_i}\leq \delta <1$. 
\end{proof}

Let \(\mathbf{w} = (x+y_1, x+y_2, \dots, x+y_n ) \in \mathbb{R}^n\), where \(y_i\) satisfy the conditions in Proposition~\ref{proposition:constructing-interior-points-alt}. Consider any \(k\) ascending indices \(r_1 < \dots < r_k\) from \([n]\). Construct the point \(\mathbf{z}= (w_{r_1},\dots,w_{r_k}) = (x+y_{r_1},\dots,x+y_{r_k}) \in \mathbb{R}^k\). This point will also satisfy the conditions of Proposition~\ref{proposition:constructing-interior-points-alt}, and therefore \(\mathbf{z}\in \text{int}(P_0)\). This concludes the proof of Theorem~\ref{proposition:n-choose-k-points-in-one-polytope}. \end{proof}

%%%%%%%%%%%%%%%%%%%%%%%%%%%%%%%%%%%%%%%%%%%%%%%%%%%%%%%%%%%%

\end{document}